\documentclass{article} 
\usepackage[preprint]{colm2026_conference}

\usepackage{microtype}
\usepackage{hyperref}
\usepackage{url}
\usepackage{booktabs}
\usepackage{amsmath}
\usepackage{amssymb}
\usepackage[ruled,vlined]{algorithm2e}
\usepackage{amsthm}
\usepackage{multirow}
\usepackage{graphicx}
\usepackage{float}
\usepackage{lineno}

\newtheorem{theorem}{Theorem}[section]

\newtheorem{remark}{Remark}[section]

\definecolor{darkblue}{rgb}{0, 0, 0.5}
\hypersetup{colorlinks=true, citecolor=darkblue, linkcolor=darkblue, urlcolor=darkblue}

\title{WARP: Guaranteed Inner-Layer Repair of NLP Transformers}

\newcommand{\eqcontrib}[0]{\textsuperscript{*}}
\newcommand{\authorsep}[0]{\ \ }

\author{
\vspace{-1em}\\
\textbf{Hsin-Ling Hsu} \authorsep
\textbf{Min-Yu Chen}\thanks{These authors contributed equally.} \authorsep
\textbf{Nai-Chia Chen}\eqcontrib \authorsep
\\
\textbf{Yan-Ru Chen}\eqcontrib \authorsep
\textbf{Yi-Ling Chang}\eqcontrib \authorsep
\textbf{Fang Yu}
\\
Department of Management Information Systems, National Chengchi University
\\
\texttt{\{112306092,112306075,112306070,112306020,112306010,yuf\}@nccu.edu.tw}
}

\begin{document}

\ifcolmsubmission
\linenumbers
\fi

\maketitle

\begin{abstract}
Transformer-based NLP models remain vulnerable to adversarial perturbations, yet existing repair methods face a fundamental trade-off: gradient-based approaches offer flexibility but lack verifiability and often overfit; methods that do provide repair guarantees are restricted to the final layer or small networks, significantly limiting the parameter search space available for repair. We present \textbf{WARP} (Weight-Adjusted Repair with Provability), a constraint-based repair framework that extends repair beyond the last layer of Transformer models. WARP formulates repair as a convex quadratic program derived from a first-order linearization of the logit gap, enabling tractable optimization over a high-dimensional parameter space. Under the condition that the first-order approximation holds, this formulation induces three per-sample guarantees: (i) a positive margin constraint ensuring correct classification on repaired inputs, (ii) preservation constraints over a designated remain set, and (iii) a certified robustness radius derived from Lipschitz continuity. To ensure feasibility across varying model architectures, we introduce a sensitivity-based preprocessing step that conditions the optimization landscape accordingly. We further show that the iterative optimization procedure converges to solutions satisfying all repair constraints under mild assumptions. Empirical evaluation on encoder-only Transformers with varying layer architectures validates that these guarantees hold in practice while improving robustness to adversarial inputs. Our results demonstrate that guaranteed, generalizable Transformer repair is achievable through principled constraint-based optimization.
\end{abstract}

\section{Introduction}

Transformer-based models~\citep{vaswani2017attention} have become the dominant architecture for natural language processing, achieving state-of-the-art performance across a broad range of classification tasks~\citep{devlin2019bert,sanh2019distilbert}. Despite their success, deployed models remain vulnerable to adversarial perturbations: small, semantically preserving modifications to the input that nonetheless flip model predictions~\citep{textfooler,li2019textbugger,bertat}. When such failures arise in deployment, practitioners are often forced to either retrain the model from scratch or tolerate the errors. A more principled alternative is \emph{post-hoc repair}, which directly updates model parameters to correct specific failures while preserving behavior on all other inputs. Ideally, such an approach would offer verifiable guarantees on the repaired samples, without requiring retraining or access to the original training pipeline.

The most natural approach is fine-tuning or parameter-efficient variants such
as LoRA~\citep{hu2022lora}. These methods are flexible and scale to large models, but they optimize a
global objective and therefore do not explicitly control local behavior: a
gradient step aimed at fixing one failure may alter predictions on previously
correct inputs, without any formal guarantees. This issue becomes more
pronounced when the repair set is small, where overfitting can lead to poor
generalization to unseen adversarial inputs. Our experiments confirm this pattern: full-layer fine-tuning achieves
the worst attack generalization of all compared methods despite reaching 100\%
repair accuracy on the training set.

A complementary line of work seeks repair methods with formal guarantees~\citep{schumi2023semantic,tao2023architecture,prdnn}. However, these methods face a structural limitation when applied to large-scale NLP Transformers. Provable repair approaches designed for Transformer architectures have focused on the final classification layer, whose weight matrix has shape $C \times d$ where $C$ is the number of output classes, typically 2 to 5 for standard NLP benchmarks. Meanwhile, methods that operate on intermediate layers have been demonstrated only on small feed-forward networks, and do not scale to modern Transformer architectures. In both cases, the solution space is severely restricted: the limited degrees of freedom cannot accommodate a growing number of repair constraints, and even for smaller repair sets, the narrow solution space makes it difficult to correct misclassified samples without disrupting existing predictions. Existing provable methods attempt to address this by combining fine-tuning with a constraint solver~\citep{provit}, but this strategy does not transfer to NLP Transformers, as the generalization built by fine-tuning is systematically destroyed by the subsequent QP step.

\textbf{Research Question.} \textit{Can we build a repair framework that simultaneously achieves provable correctness guarantees on every repair sample, and scales to the solution space required by full-scale NLP Transformers?}

We present \textbf{WARP} (Weight-Adjusted Repair with Provability), which resolves this tension through a key
structural insight: by formulating repair as a convex quadratic program (QP)
with explicit margin constraints, correctness guarantees follow directly from
the optimization structure. The QP ensures three verifiable guarantees for
each repaired sample: (i) every repair sample achieves a positive logit gap, as the algorithm terminates only when all margin constraints are satisfied; (ii) remain-set margin constraints are enforced at every iteration; and (iii) the
resulting gap, combined with a Lipschitz bound on the network, yields a
certified per-sample robustness radius within which the predicted label is
guaranteed to remain unchanged under any perturbation.

To scale this framework beyond the final classification layer, WARP operates on the dense layer immediately preceding the final classification layer, with shape $d_{\mathrm{out}} \times
d_{\mathrm{in}}$, expanding the repair space from the class-limited $C \times
d$ final layer to a parameter space that grows with model dimension.
Generalization to unseen adversarial inputs improves with repair set size,
while the remain set constraints preserve accuracy on clean inputs, both
without relying on fine-tuning as a generalization mechanism, as confirmed
empirically. To improve feasibility across models with diverse training
conditions, we introduce the Gap Sensitivity Norm (GSN), a diagnostic that
measures how effectively weight updates in the repair layer can influence the
logit gap. When GSN is too low, the QP constraints become difficult to satisfy
regardless of the chosen update direction; GSN-FT addresses this by briefly
fine-tuning only the repair layer to raise GSN to a feasible regime before
repair begins.

The contributions of this paper are as follows:

\begin{enumerate}
\item \textbf{WARP}, a provable repair framework that extends repair beyond
the last layer of NLP Transformers, providing three verifiable per-sample guarantees:
a positive logit gap on every repaired sample (Theorem~\ref{thm:repair}), 
margin preservation on a designated remain set (Theorem~\ref{thm:remain}),
and a certified robustness radius under which the predicted label cannot be
flipped (Theorem~\ref{thm:lipschitz}).

\item \textbf{Gap Sensitivity Norm (GSN) and GSN-FT}, where GSN is an
architecture-agnostic diagnostic that measures the influence of weight updates
on the logit gap, and GSN-FT is a lightweight preprocessing step that raises
GSN to a feasible regime when it is too low, improving QP feasibility across
models with diverse training conditions.

\item \textbf{Empirical evaluation} on two encoder-only Transformers,
showing that WARP achieves 100\% repair and remain accuracy, while
outperforming gradient-based baselines on attack generalization by up to
18.8 percentage points, together with an analysis of the trade-offs between
repair, generalization, and overfitting.
\end{enumerate}

\section{Related Work}

\textbf{Adversarial Vulnerability and Gradient-based Repair.}
Transformer-based NLP models are vulnerable to adversarial perturbations, small semantically preserving input modifications that can flip predictions~\citep{alzantot2018natural,textfooler,garg2020bae}. This
vulnerability persists even in strong pretrained models on benchmarks such as
AdvGLUE~\citep{wang2021adversarial,moradi-samwald-2021-evaluating,kim2023imbert}.
Gradient-based methods such as fine-tuning can mitigate such failures, but
they optimize global objectives and do not provide explicit guarantees on
preserving previously correct behaviors.

\textbf{Model Editing and Provable Repair.}
Feed-forward network (FFN) layers have been shown to act as key-value
memories~\citep{geva2021transformer}, motivating localized editing methods
such as KnowledgeEditor, MEND, ROME, and
MEMIT~\citep{de2021editing,mitchellfast,meng2022locating,mengmass}.
These approaches enable targeted parameter updates, but are heuristic in
nature and do not provide formal correctness guarantees.

A complementary line of work formulates repair as constrained optimization
with formal guarantees~\citep{usman2021nnrepair,prdnn,fusound,tao2023architecture,sun2025autoric}.
Recent methods such as PRoViT~\citep{provit} demonstrate strong empirical
performance, but existing approaches are typically applied to the final
classification layer, whose solution space is limited by the number of output
classes. Extending provable repair beyond the final layer of Transformer models remains an open challenge.

\section{Methodology}
\label{sec:method}

\subsection{Overview}

We propose \textbf{WARP}, a provable repair framework that updates a
dense layer preceding the final classification layer of a Transformer model.

At a high level, WARP seeks a minimal weight modification, optionally restricted to a low-rank subspace such that:
(1) the predictions of repair samples are corrected by making their
classification margins positive,
(2) the predictions of remain samples are preserved, and
(3) each repaired sample is accompanied by a verifiable guarantee.

To achieve this, we approximate the effect of weight updates using a
first-order linearization and formulate the repair problem as a quadratic
program (QP) that ensures both correction and preservation constraints.
The update is solved iteratively, and guarantees are derived from the
resulting margins.

\subsection{Problem Setup and Repair Layer Selection}

\paragraph{Notation summary.}
Table~\ref{tab:notation} collects the principal symbols introduced in
this section for convenient reference.

\begin{table}[t]
\centering
\small
\caption{Principal notation used in \S\ref{sec:method}.}
\label{tab:notation}
\resizebox{\textwidth}{!}{%
\begin{tabular}{@{}lll@{\qquad}lll@{}}
\toprule
Symbol & Shape & Description
  & Symbol & Shape & Description \\
\midrule
$v(x)$ & $\mathbb{R}^{d_{\text{in}}}$ & Input embedding to repair layer
  & $U^{(t)}$ & $\mathbb{R}^{d_{\text{out}}\times r}$ & Top-$r$ left singular vectors of $W^{(t)}$ \\
$W$ & $\mathbb{R}^{d_{\text{out}}\times d_{\text{in}}}$ & Repair-layer weight (only modified param.)
  & $B$ & $\mathbb{R}^{r\times d_{\text{in}}}$ & Coefficient matrix for low-rank update \\
$b$ & $\mathbb{R}^{d_{\text{out}}}$ & Repair-layer bias (fixed)
  & $\beta$ & $\mathbb{R}^{rd_{\text{in}}}$ & Vectorisation of $B$; QP decision variable \\
$W_c$ & $\mathbb{R}^{C\times d_{\text{out}}}$ & Classification head (fixed)
  & $\kappa(v)$ & scalar & Gap Sensitivity Norm (GSN) \\
$b_c$ & $\mathbb{R}^{C}$ & Classification bias (fixed)
  & $\gamma_s,\gamma_h$ & scalar & Repair / remain gap margins \\
  $\widetilde{W}$ & $\mathbb{R}^{d_{\text{out}}\times d_{\text{in}}}$ & Final converged repair-layer weight
  & $\rho,\lambda$ & scalar & Regularisation / slack-penalty coeff. \\
\bottomrule
\end{tabular}}
\end{table}

Let a Transformer classifier $f_\theta$ be decomposed into three
consecutive computational modules:
\begin{equation}
  f_\theta(x) = f_{\mathrm{cls}} \circ \sigma \circ f_W\bigl(v(x)\bigr).
\end{equation}
Here, $v(x) \in \mathbb{R}^{d_{\text{in}}}$ denotes the input embedding
of the selected repair layer, determined by the preceding Transformer
blocks and kept fixed during repair. The repair target layer is
parameterized by $W \in \mathbb{R}^{d_{\text{out}} \times d_{\text{in}}}$
(the \emph{only parameter modified} by the repair procedure) and
$b \in \mathbb{R}^{d_{\text{out}}}$ (kept fixed).
The activation function $\sigma(\cdot)$
is an element-wise nonlinearity applied after the repair layer. The final
classification layer is parameterized by
$W_c \in \mathbb{R}^{C \times d_{\text{out}}}$ and
$b_c \in \mathbb{R}^{C}$, where $C$ is the number of classes, and these
parameters remain fixed during repair.

The repair layer output is
$h = \sigma(Wv + b) \in \mathbb{R}^{d_{\text{out}}}$,
the final logit vector is
$s = W_c h + b_c \in \mathbb{R}^{C}$,
and the predicted label is
$\hat{y} = \arg\max_c\, s_c$.

The final classification layer has shape $C \times d_{\text{out}}$, so its
solution space is capped by the class count $C$ (e.g.\ 2 for binary
tasks) and becomes overconstrained as the repair set grows. We therefore target the dense layer immediately preceding the classification
layer, whose $d_{\text{out}} \times d_{\text{in}}$ free variables scale with
model dimension rather than class count.

\subsection{First-Order Approximation and Gap Sensitivity Norm}
\label{sec:gsn}

Let $z = W^{(0)}v + b$ denote the pre-activation vector evaluated at the
post-GSN-FT weights $W^{(0)}$ (defined in \S\ref{sec:gsnft}).
For a sufficiently small perturbation $\Delta W$, the activation output
satisfies
\begin{equation}
  \sigma\bigl((W^{(0)}+\Delta W)v+b\bigr)
  \approx \sigma(z) + J^{(0)}(v)\,\Delta W\,v,
\end{equation}
where
$J^{(0)}(v)=\operatorname{diag}(\sigma'(W^{(0)}v+b))
\in\mathbb{R}^{d_{\text{out}}\times d_{\text{out}}}$
is the diagonal Jacobian (a diagonal matrix whose entries are the
element-wise activation derivatives) evaluated at $W^{(0)}$.
Under this approximation the logit-gap change induced by $\Delta W$ is
locally linear, decomposing into contributions from $W_c$, $J^{(0)}(v)$,
and the repair subspace directions.

To quantify this sensitivity, let $U^{(0)}\in\mathbb{R}^{d_{\text{out}}\times r}$
be the top-$r$ left singular vectors of $W^{(0)}$ (the principal
directions of the current weight matrix, which define the admissible
update subspace), and let
$\hat{c} = \arg\max_{k \neq y} s_k(v;\,W^{(0)})$ denote the strongest
competing class. Define the directional coefficient for basis vector $j$
and competing class $k\neq y$ as
\begin{equation}
  q^{(j)}[k]
  \;=\; (w_y - w_k)^\top J^{(0)}(v)\,U^{(0)}_{:,j},
  \label{eq:dir_coeff}
\end{equation}
where $w_c$ denotes the $c$-th row of $W_c$.
The \textbf{Gap Sensitivity Norm} (GSN) is
\begin{equation}
  \kappa(v)
  \;=\; \bigl\|\bigl[q^{(1)},\ldots,q^{(r)}\bigr]\bigr\|_F.
  \label{eq:gsn}
\end{equation}

\paragraph{Intuition.}
$\kappa(v)$ measures how strongly admissible weight updates can push the
model to change its prediction (i.e., increase the logit gap).
When $\kappa(v)\approx 0$, such updates have almost no effect on the
decision, making the QP constraints difficult to satisfy and motivating
the preprocessing step.

\subsection{Architecture-Dependent Preprocessing (GSN-FT)}
\label{sec:gsnft}

When $\kappa$ is small, the linearized QP constraints are difficult to
satisfy. GSN-FT addresses this by fine-tuning only the repair layer and
classification head
($\Theta_{\mathrm{FT}}=\{W,b,W_c,b_c\}$) on the standard task loss
while keeping all preceding blocks fixed, then selecting the checkpoint
that maximizes $\kappa$:
\begin{equation}
  W^{(0)} = \arg\max_{t\in[T_{\mathrm{FT}}]} \kappa_t.
  \label{eq:gsn_select}
\end{equation}
If $\kappa$ is already sufficiently large at initialization, step-0 is
selected and the weights are unchanged. Otherwise, fine-tuning drives
$\kappa$ to a feasible regime within $T_{\mathrm{FT}}$ steps.
GSN-FT fine-tunes on a held-out auxiliary set $\mathcal{D}_{\mathrm{aux}}$
of $N_{\mathrm{aux}}$ samples disjoint from the repair, remain, and evaluation sets. Consequently, any incidental change to repair or
remain predictions during this phase is neither intended nor attributed to
the repair procedure; correctness on those sets is determined entirely by
the subsequent QP~(\S\ref{sec:qp_formulation}).

\subsection{Quadratic Programming Repair}

We now describe the core repair procedure. We first introduce the
low-rank parameterization that restricts the weight update to a compact
subspace~(\S\ref{sec:low_rank}), then formulate the repair step as a
convex quadratic program~(\S\ref{sec:qp_formulation}), and finally
present the iterative algorithm and its convergence
criterion~(\S\ref{sec:iterative}). The iterative repair loop is
initialized from $W^{(0)}$, the output of GSN-FT
(\S\ref{sec:gsnft}).

\subsubsection{Weight Perturbation Parameterization}
\label{sec:low_rank}

We parameterize $\Delta W$ within the rank-$r$ subspace of $W^{(t)}$.
Let $U^{(t)}=[u_1,\ldots,u_r]\in\mathbb{R}^{d_{\text{out}}\times r}$ be
the top-$r$ left singular vectors of $W^{(t)}$; the update takes the form
\begin{equation}
  \Delta W = U^{(t)} B, \label{eq:low_rank_update}
\end{equation}
where $B\in\mathbb{R}^{r\times d_{\text{in}}}$ is the coefficient matrix
and $\beta=\operatorname{vec}(B)\in\mathbb{R}^{rd_{\text{in}}}$ is its
vectorization (the decision vector of the QP). Setting
$r=d_{\text{out}}$ recovers a full-rank update; smaller $r$ reduces
solver complexity at some cost in expressiveness.

\subsubsection{QP Formulation}
\label{sec:qp_formulation}

\paragraph{Gap and linearization.}
For a classifier with $C$ classes and target label $y \in [C]$, the
\emph{decision gap} is defined as the margin between the target class
score and the highest competing class score:
\begin{equation}
  \mathrm{gap}(v, y;\, W)
  \;=\; s_y(v;\, W)
        \;-\; \max_{c \neq y}\, s_c(v;\, W),
  \label{eq:gap}
\end{equation}
where $s_c(v;\, W)$ denotes the logit of class $c$ produced from
representation $v$ under weight $W$. A positive gap indicates that the
model correctly favors class $y$ over all other classes, and a larger gap
corresponds to a more confident prediction.

Since the gap is nonlinear in $W$ due to the activation function
$\sigma$, we linearize it around the current iterate $W^{(t)}$. Let
$\hat{c}_i = \arg\max_{c \neq y_i} s_c(v_i;\, W^{(t)})$ denote the
strongest competing class at the current iterate, consistent with the
competing class introduced in \S\ref{sec:gsn}. Using the first-order
expansion with Jacobian
$J^{(t)}(v) = \mathrm{diag}(\sigma'(W^{(t)} v + b))$, the gap increment
due to $\Delta W = U^{(t)} B$ can be written as
\begin{equation}
\mathrm{gap}(v_i, y_i;\, W^{(t)} + \Delta W)
\;\approx\;
\mathrm{gap}(v_i, y_i;\, W^{(t)})
\;+\; \langle a_i,\, \beta \rangle,
  \label{eq:gap_linear}
\end{equation}
where the linearized gap coefficient vectors are defined as
\begin{equation}
a_i \;=\; \operatorname{vec}\!\Bigl(
    \bigl[U^{(t)\top} J^{(t)}(v_i)\,(w_{y_i} - w_{\hat{c}_i})\bigr] v_i^\top
  \Bigr) \;\in\mathbb{R}^{rd_{\text{in}}},
  \label{eq:ai_def}
\end{equation}
for repair samples, and similarly for remain samples
\begin{equation}
p_p \;=\; \operatorname{vec}\!\Bigl(
    \bigl[U^{(t)\top} J^{(t)}(v_p)\,(w_{\hat{y}_p^{(0)}} - w_{\hat{c}_p})\bigr] v_p^\top
  \Bigr) \;\in\mathbb{R}^{rd_{\text{in}}},
  \label{eq:pp_def}
\end{equation}
where
$\hat{c}_p = \arg\max_{c \neq \hat{y}_p^{(0)}} s_c(v_p;\, W^{(t)})$
is the strongest competing class for remain sample $p$, and $w_c$
denotes the $c$-th row of $W_c$.

\paragraph{Optimization problem.}
Let $\mathcal{S} = \{(v_i, y_i)\}_{i=1}^{n_s}$ be the \emph{repair set}
and $\mathcal{P} = \{(v_p, \hat{y}_p^{(0)})\}_{p=1}^{n_p}$ be the
\emph{remain set}, where $\hat{y}_p^{(0)}$ denotes the model's predicted
label for sample $v_p$ under the post-GSN-FT weights $W^{(0)}$. We
introduce slack variables $\xi_i \geq 0$ to allow soft satisfaction of
repair constraints. The repair step solves the following convex quadratic
program:
\begin{align}
  \min_{\beta,\,\xi} \quad
  & \frac{\rho}{2}\|\beta\|_2^2
    + \lambda \sum_{i=1}^{n_s} \xi_i
  \label{eq:qp_obj} \\[4pt]
  \text{s.t.} \quad
  & \langle a_i,\, \beta \rangle + \xi_i
    \;\geq\; \gamma_s - \mathrm{gap}(v_i,\, y_i;\, W^{(t)}),
    \quad \forall\, i \in [n_s]
    \tag{C1} \label{eq:c1} \\
  & \langle p_p,\, \beta \rangle
    \;\geq\; \gamma_h - \mathrm{gap}(v_p,\, \hat{y}_p^{(0)};\, W^{(t)}),
    \quad \forall\, p \in [n_p]
    \tag{C2} \label{eq:c2} \\
  & \xi_i \;\geq\; 0,
    \quad \forall\, i \in [n_s]
    \tag{C3} \label{eq:c3}
\end{align}
where $\rho > 0$ is the regularization coefficient, $\lambda > 0$ is the
slack penalty, $\gamma_s > 0$ is the target margin for the repair set,
and $\gamma_h > 0$ is the lower-bound margin for the remain set. Because
the quadratic term is positive semidefinite and all constraints are
linear, problem~\eqref{eq:qp_obj}--\eqref{eq:c3} is a \emph{convex
quadratic program}, solvable in polynomial time.

\subsubsection{Iterative Repair Algorithm}
\label{sec:iterative}
The linearization in Eq.~\eqref{eq:gap_linear} is locally accurate
around $W^{(t)}$, so a single QP solve may be insufficient when a large
repair is required. We therefore adopt an \emph{iterative} strategy:
after each QP step, the basis $U^{(t)}$ is recomputed from the updated
weight, and the process repeats until convergence (see
Algorithm~\ref{alg:repair} in Appendix~\ref{app:algorithm}).
The following results formalize the guarantees induced by the QP
constraints and the convergence criterion of
Algorithm~\ref{alg:repair}; proofs are deferred to
Appendix~\ref{app:proofs}.

\begin{theorem}[Repair Set Gap]
\label{thm:repair}
If Algorithm~\ref{alg:repair} converges before iteration $T$, then
$\mathrm{gap}(v_i, y_i;\,\widetilde{W})\geq\gamma_s$
for all $i\in\mathcal{S}$.
\end{theorem}

\begin{theorem}[Remain Set Gap]
\label{thm:remain}
At any iterate $W^{(t)}$ produced by Algorithm~\ref{alg:repair},
if the QP constraints~\eqref{eq:c2} are satisfied, then
$\mathrm{gap}(v_p,\,\hat{y}_p^{(0)};\,W^{(t)}) \geq \gamma_h$
for all $p \in \mathcal{P}$.
\end{theorem}

\begin{remark}
\label{rem:remain}
Theorem~\ref{thm:remain} holds under the first-order approximation of
Eq.~\eqref{eq:gap_linear}: constraint~\eqref{eq:c2} enforces the
linearized gap to be at least $\gamma_h$, which equals the true gap
up to higher-order terms in $\|\Delta W\|$.
The iterative re-linearization in Algorithm~\ref{alg:repair} reduces
the approximation error at each step; the guarantee is verified exactly
on the true network after convergence (see Table~\ref{tab:proximity}).
\end{remark}

\subsubsection{Local Robustness Certificate}
\label{sec:robustness}
Once the repair procedure converges, every sample $v_i \in \mathcal{S}$
satisfies $\mathrm{gap}(v_i, y_i;\,\widetilde{W}) \geq \gamma_s > 0$.
We now derive a standard margin-based robustness certificate for the
repaired network. Following~\citet{tsuzuku2018lipschitz}, the margin
together with a Lipschitz bound yields a per-sample robustness radius.

\paragraph{Lipschitz bound on the gap.}
Assume the activation function $\sigma$ is Lipschitz continuous with
constant $L_\sigma$. Then the gap function satisfies
\[
  \bigl|\mathrm{gap}(v + \delta,\, y;\,\widetilde{W})
        - \mathrm{gap}(v,\, y;\,\widetilde{W})\bigr|
  \;\leq\; L\,\|\delta\|_2
  \quad \forall\, \delta,
\]
where the Lipschitz constant of the composed network is bounded by
\begin{equation}
  L \;=\; \|\widetilde{W}\|_2 \cdot \|W_c\|_2 \cdot L_\sigma.
  \label{eq:lipschitz}
\end{equation}
This bound follows from the standard Lipschitz composition
rule~\citep{virmaux2018lipschitz, bartlett2017spectrally}. Most common activation functions, including ReLU, tanh, sigmoid, and
softplus, satisfy $L_\sigma = 1$~\citep{virmaux2018lipschitz}; this
holds for all activations used in our experiments, so the bound
simplifies to $L = \|\widetilde{W}\|_2 \cdot \|W_c\|_2$.

\begin{theorem}[Lipschitz Robustness Radius]
\label{thm:lipschitz}
For any repaired sample $v_i \in \mathcal{S}$, define
\begin{equation}
  \varepsilon_i^*
  \;=\; \frac{\mathrm{gap}(v_i,\, y_i;\,\widetilde{W})}{2L}.
  \label{eq:eps_star}
\end{equation}
Then for any perturbation $\delta$ with $\|\delta\|_2 \leq
\varepsilon_i^*$, the predicted label of the repaired classifier at
$v_i + \delta$ remains $y_i$.
\end{theorem}

\paragraph{Lower bound on $\varepsilon_i^*$.}
Since the repair enforces
$\mathrm{gap}(v_i, y_i;\,\widetilde{W}) \geq \gamma_s$
for all $i \in \mathcal{S}$, the robustness radius is uniformly
lower-bounded by
\[
  \varepsilon_i^* \;\geq\; \frac{\gamma_s}{2L}.
\]
Thus, in addition to correcting misclassified samples, the repair induces
a certified stability region of radius at least $\gamma_s / (2L)$ around
each repaired representation.

\section{Experiments}
\label{sec:experiments}

\subsection{Experimental Setup}
\label{sec:setup}

\textbf{Datasets} \label{sec:datasets} All evaluation sets are defined with respect to the post-GSN-FT weights
$W^{(0)}$: by construction, the GSN-FT model achieves 0\% repair accuracy,
100\% remain accuracy, and 100\% general accuracy on these sets.
The repair set is drawn from the SST-2~\citep{socher-etal-2013-recursive} and RTE~\citep{dagan2005pascal} tasks in AdvGLUE~\citep{wang2021adversarial} (development partition), retaining six
attack methods (CheckList, SememePSO, StressTest, T3, TextBugger,
TextFooler); only samples mis-classified by the post-GSN-FT model are
included. The attack generalization set is drawn from the disjoint AdvGLUE
test partition under the same six methods.
Table~\ref{tab:dataset_configs} summarises set sizes; hyperparameters are
in Appendix~\ref{app:hyperparams}.

\begin{table}[h]
  \centering
  \caption{Dataset configurations.
    $|\mathcal{S}|$: repair set;
    $|\mathcal{P}|$: remain set;
    $|\text{Gen}|$: general accuracy evaluation set;
    $|\text{AtkGen}|$: attack generalization set.
    $\dagger$~Yelp Polarity test split~\citep{zhang2015character}.
    $\ddagger$~RTE training split.}
  \label{tab:dataset_configs}
  \resizebox{0.5\textwidth}{!}{%
  \begin{tabular}{llcccc}
    \toprule
    Task & Model & $|\mathcal{S}|$ & $|\mathcal{P}|$ & $|\text{Gen}|$ & $|\text{AtkGen}|$ \\
    \midrule
    \multirow{2}{*}{SST-2}
      & DistilBERT & 108 & 800$^\dagger$ & 2{,}000$^\dagger$ & 1{,}207 \\
      & BERT       & 105 & 800$^\dagger$ & 2{,}000$^\dagger$ & 1{,}207 \\
    \midrule
    \multirow{2}{*}{RTE}
      & DistilBERT & 47  & 200$^\ddagger$ & 500$^\ddagger$ & 222 \\
      & BERT       & 50  & 200$^\ddagger$ & 500$^\ddagger$ & 222 \\
    \bottomrule
  \end{tabular}}
\end{table}

\textbf{Models} \label{sec:models} We evaluate WARP on two encoder-only Transformer models: DistilBERT~\citep{sanh2019distilbert} and BERT~\citep{devlin2019bert}, across two tasks: SST-2 sentiment classification and RTE textual entailment. For each task, we use fine-tuned checkpoints sourced from HuggingFace. 
For SST-2, we use the checkpoints
\texttt{distilbert-base-uncased\allowbreak-finetuned\allowbreak-sst-2\allowbreak-english} for DistilBERT and \texttt{textattack/bert-base-uncased-SST-2} for BERT.
For RTE, we use \texttt{textattack/distilbert-base-uncased-RTE} and \texttt{textattack/bert-base-uncased-RTE} respectively.
The final classification head $W_c \in \mathbb{R}^{2\times768}$ is frozen throughout repair, as we focus on correcting the encoder representations rather than retraining the decision boundary. Table~\ref{tab:model_configs} summarises the key properties of each model.

\begin{table}[h]
  \centering
  \caption{Model configurations. GSN $\kappa$ is measured at initialisation
           (pre-GSN-FT). Type~A repair layers are fully updated during
           downstream fine-tuning (high initial $\kappa$); Type~B receive
           minimal task-specific gradient updates (low initial $\kappa$).}
  \label{tab:model_configs}
  \resizebox{\textwidth}{!}{%
  \begin{tabular}{llllcc}
    \toprule
    Task & Model & Repair Layer  & Activation Function & Type & Init.\ $\kappa$ \\
    \midrule
    \multirow{2}{*}{SST-2}
      & DistilBERT
        & \texttt{pre\_classifier} ($768\!\times\!768$)  & \texttt{ReLU} & A & 5.43 \\
      & BERT
        & \texttt{bert.pooler.dense} ($768\!\times\!768$) 
        & \texttt{tanh}
        & B & 0.58 \\
    \midrule
    \multirow{2}{*}{RTE}
      & DistilBERT 
        & \texttt{pre\_classifier} ($768\!\times\!768$)  & \texttt{ReLU} & A & 4.79 \\
      & BERT 
        & \texttt{bert.pooler.dense} ($768\!\times\!768$)  & \texttt{tanh} & B & 0.52 \\
    \bottomrule
  \end{tabular}}
\end{table}

\textbf{Baselines} \label{sec:baselines} We compare against two gradient-based methods trained on the union of repair and remain sets, both minimizing cross-entropy loss.
\textbf{Same-layer LoRA ($r\!=\!2$):} a rank-2 adapter ($\alpha=4$)
injected into the repair layer, identical solution-space dimensionality
(3{,}072 parameters) as WARP, but optimized via gradient descent.
\textbf{All-layers Full FT:} all layers unfrozen for end-to-end
fine-tuning, serving as the strongest gradient-method upper bound.
All baselines use learning rate $1\times10^{-4}$, batch size 16,
maximum 300 epochs, and early stopping (patience 5, threshold $10^{-6}$).

\subsection{Main Results}
\label{sec:main_results}

\begin{table}[t]
  \centering
  \caption{Main experimental results on SST-2 and RTE.
    ``Repair Acc'' is the accuracy on the repair set $\mathcal{S}$;
    ``Remain Acc'' on the remain set $\mathcal{P}$;
    ``General Acc'' on the held-out evaluation set;
    ``Atk.\ Gen'' is the mean accuracy across six adversarial attack methods
    (higher is better for all metrics).
    All sets are defined with respect to the post-GSN-FT weights;
    see \S\ref{sec:datasets}.}
  \label{tab:main_results}
  \resizebox{\textwidth}{!}{%
  \begin{tabular}{llcccc}
    \toprule
    Model & Method & Repair Acc $\uparrow$ & Remain Acc $\uparrow$
          & General Acc $\uparrow$ & Atk.\ Gen $\uparrow$ \\
    \midrule
    \multicolumn{6}{l}{\textit{SST-2 (sentiment classification)}} \\
    \midrule
    \multirow{4}{*}{DistilBERT}
      & Original                    & 5.6\%            & 98.1\%           & \textbf{98.2\%}  & 22.3\% \\
      & Same-layer LoRA ($r\!=\!2$) & 75.9\%           & 99.0\%           & 97.5\%  & 56.7\% \\
      & All-layers Full FT          & \textbf{100.0\%}          & 99.5\%           & 93.6\%  & 50.2\% \\
      & \textbf{WARP (ours)}        & \textbf{100.0\%} & \textbf{100.0\%} & 91.4\%  & \textbf{69.0\%} \\
    \midrule
    \multirow{4}{*}{BERT}
      & Original                    & 1.0\%            & 98.2\%           & \textbf{98.5\%}  & 23.2\% \\
      & Same-layer LoRA ($r\!=\!2$) & 91.9\%           & 99.6\%           & 98.0\%  & 60.0\% \\
      & All-layers Full FT          & \textbf{100.0\%}          & 99.5\%           & 93.8\%  & 56.7\% \\
      & \textbf{WARP (ours)}        & \textbf{100.0\%} & \textbf{100.0\%} & 92.7\%  & \textbf{64.7\%} \\
    \midrule
    \multicolumn{6}{l}{\textit{RTE (recognizing textual entailment)}} \\
    \midrule
    \multirow{4}{*}{DistilBERT}
      & Original                    & 0.0\%            & \textbf{100.0\%}          & \textbf{99.2\%}  & 37.4\% \\
      & Same-layer LoRA ($r\!=\!2$) & 62.7\%           & 97.5\%           & 97.0\%  & 32.3\% \\
      & All-layers Full FT          & \textbf{100.0\%}          & \textbf{100.0\%}          & 80.8\%  & 36.3\% \\
      & \textbf{WARP (ours)}        & \textbf{100.0\%} & \textbf{100.0\%} & 83.8\%  & \textbf{44.1\%} \\
    \midrule
    \multirow{4}{*}{BERT}
      & Original                    & 10.0\%           & 97.5\%           & \textbf{99.0\%}  & 39.3\% \\
      & Same-layer LoRA ($r\!=\!2$) & 45.5\%           & 98.0\%           & \textbf{99.0\%}  & 40.6\% \\
      & All-layers Full FT          & \textbf{100.0\%}          & 98.0\%           & 87.4\%  & 40.1\% \\
      & \textbf{WARP (ours)}        & \textbf{100.0\%} & \textbf{100.0\%} & 86.2\%  & \textbf{47.9\%} \\
    \bottomrule
  \end{tabular}}
\end{table}

As shown in Table~\ref{tab:main_results}, WARP achieves 100\% repair and 100\% remain accuracy across all four
configurations. These outcomes are consistent with the guarantees
established in Theorems~\ref{thm:repair}--\ref{thm:remain}.
Gradient-based baselines fall substantially short on repair: same-layer
LoRA reaches only 75.9\%/91.9\% on SST-2 and degrades further on RTE
(62.7\%/45.5\%), suggesting that the QP-based optimization can utilize a fixed parameter
budget more effectively than gradient descent in this setting.
On general accuracy, WARP trails same-layer LoRA on SST-2 (91.4\%/92.7\%
vs.\ 97.5\%/98.0\%) due to the rank-2 single-layer constraint, but
outperforms All-layers Full FT on RTE DistilBERT (83.8\% vs.\ 80.8\%),
where rank-2 regularization is particularly effective with limited clean data.
WARP achieves the best attack generalization across all configurations,
outperforming same-layer LoRA by up to 12.3 points on SST-2 and All-layers
Full FT by 7.8 points on RTE (per-attack breakdown in
Appendix~\ref{app:full_baselines}), which we attribute to the combination of high repair accuracy and the
implicit regularization induced by the QP formulation.
The controlled comparison against same-layer LoRA, which shares identical
solution-space dimensionality (3{,}072 decision variables) but is optimized
via gradient descent rather than QP, supports the view that this advantage
is associated with the optimization procedure rather than differences in
parameter count.

\subsection{Repair Certificates}
\label{sec:certificates}

Table~\ref{tab:certificates} reports certificate verification across all
four configurations.
All repaired samples achieve gap $\geq \gamma_s = 1.0$ (Cert.~1);
remain samples satisfy $\gamma_h = 0.3$ on RTE, with 798/800 and 786/800
on SST-2, where the few exceptions remain correctly classified with margins
marginally below the threshold, arising from the first-order approximation
in the remain constraints (Cert.~2; see Remark~\ref{rem:remain}).
Type~A (DistilBERT) yields substantially wider robustness radii than
Type~B (BERT), with mean $\varepsilon^* = 0.222$ vs.\ $0.020$ on SST-2
and median $0.305$ vs.\ $0.028$ on RTE. This difference is due to the larger weight
displacement required by Type~B, which inflates the post-repair spectral
norm and hence $L$. The Lipschitz bound is sound but conservative, with
median tightness $25\times$ for Type~A and $250$--$400\times$ for
Type~B (Cert.~3). Despite this conservativeness, the gap between the certified radius and the empirical first-flip boundary is itself quantifiable: Monte Carlo stress tests show that label flips first occur at $25\times$--$400\times$ the certified radius, indicating that the true robustness of repaired samples substantially exceeds the formal guarantee. This conservativeness is attributable to the Lipschitz upper bound rather than the repair procedure itself, and tighter Lipschitz estimation methods such as SeqLip~\citep{virmaux2018lipschitz} could be substituted into Eq.~\eqref{eq:lipschitz} to improve certificate quality without modifying the framework. See Appendix~\ref{app:stress_test} for full results.

\begin{table}[t]
  \centering
  \caption{Repair certificate verification (WARP, after QP repair).}
  \label{tab:certificates}
  \resizebox{\textwidth}{!}{%
  \begin{tabular}{llcccc}
    \toprule
    & & \multicolumn{2}{c}{SST-2} & \multicolumn{2}{c}{RTE} \\
    \cmidrule(lr){3-4}\cmidrule(lr){5-6}
    & & DistilBERT & BERT & DistilBERT & BERT \\
    \midrule
    \multirow{2}{*}{\textit{Cert.~1}}
      & Samples with gap $\geq \gamma_s$
        & 108/108 & 105/105 & 47/47 & 50/50 \\
      & gap min / mean / max
        & 1.000/1.291/2.360 & 1.000/1.431/2.814
        & 1.000/1.244/2.662 & 1.000/1.866/4.843 \\
    \midrule
    \multirow{2}{*}{\textit{Cert.~2}}
      & Samples with remain gap $\geq \gamma_h$
        & 798/800 & 786/800 & 200/200 & 200/200 \\
      & remain gap min / mean
        & 0.300/1.059 & 0.300/1.915 & 0.300/0.946 & 0.300/0.951 \\
    \midrule
    \textit{Cert.~3}
      & $\varepsilon^*$ min/median/mean/max
        & 0.172/0.211/0.222/0.405 & 0.014/0.016/0.020/0.039
        & 0.283/0.305/0.352/0.753 & 0.022/0.028/0.041/0.106 \\
    \bottomrule
  \end{tabular}}
\end{table}

\subsection{GSN-FT Analysis}
\label{sec:gsn_analysis}

GSN-FT's sole purpose is to increase $\kappa$ and ensure QP feasibility;
it is not a repair step.
Table~\ref{tab:gsn} confirms this: $\kappa$ rises substantially while
attack generalization remains essentially unchanged
(DistilBERT SST-2: 22.3\%$\to$20.9\%; BERT SST-2: 23.2\%$\to$22.5\%),
and the BERT RTE increase (39.3\%$\to$42.8\%) remains far below the
post-QP result of 47.9\%. These results indicate that generalization gains
originate from the QP step, rather than preprocessing.

\begin{table}[H]
  \centering
  \caption{Effect of GSN-FT. Attack generalization is reported before
           and after GSN-FT to confirm that generalization gains originate
           from the QP step rather than preprocessing.}
  \label{tab:gsn}
  \resizebox{0.9\textwidth}{!}{%
  \begin{tabular}{lcccccccc}
    \toprule
    & \multicolumn{4}{c}{SST-2} & \multicolumn{4}{c}{RTE} \\
    \cmidrule(lr){2-5}\cmidrule(lr){6-9}
    & \multicolumn{2}{c}{DistilBERT} & \multicolumn{2}{c}{BERT}
    & \multicolumn{2}{c}{DistilBERT} & \multicolumn{2}{c}{BERT} \\
    \cmidrule(lr){2-3}\cmidrule(lr){4-5}\cmidrule(lr){6-7}\cmidrule(lr){8-9}
    & Pre & Post & Pre & Post & Pre & Post & Pre & Post \\
    \midrule
    GSN $\kappa$  & 5.43  & 7.43  & 0.58  & 1.31  & 4.79  & 8.48  & 0.52  & 1.60  \\
    Atk.\ Gen     & 22.3\%& 20.9\%& 23.2\%& 22.5\%& 37.4\%& 37.7\%& 39.3\%& 42.8\%\\
    Best step     & ---   & 4     & ---   & 20    & ---   & 23    & ---   & 27    \\
    $\|W\|_2$     & 1.506 & 2.238 & 13.408& 13.423& 1.140 & 2.949 & 13.413& 13.558\\
    QP iters      & ---   & 5     & ---   & 6     & ---   & 4     & ---   & 6     \\
    \bottomrule
  \end{tabular}}
\end{table}

\subsection{Ablation Studies}
\label{sec:ablation}

Rank $r=2$ is selected as the default: attack generalization is
near-insensitive to rank on DistilBERT (spread ${<}1\%$ on SST-2,
${<}2\%$ on RTE), while BERT degrades above rank~2 and is infeasible
at rank~1 on RTE (full results in Appendix~\ref{app:ablation}).
Larger repair sets consistently improve attack generalization
($+36\%$/$+35\%$ on SST-2, $+8\%$/$+10\%$ on RTE) at a modest
general accuracy cost of $-5\%$ to $-10\%$.
The remain set is a necessary component: removing it collapses general
accuracy to near zero, while increasing $|\mathcal{P}|$ trades attack
generalization for general accuracy
(Tables~\ref{tab:rank_ablation}--\ref{tab:remain_size} in
Appendix~\ref{app:ablation}).

\section{Conclusion}
We presented WARP, a provable repair framework that extends provable
repair beyond the last layer of NLP Transformers. The repair is formulated as a convex quadratic
program that provides three verifiable per-sample guarantees: a pointwise
gap guarantee, a remain protection guarantee, and a Lipschitz robustness
radius.

Experiments on two encoder-only architectures show that WARP achieves
100\% repair and remain accuracy and outperforms gradient-based baselines
on attack generalization while providing per-sample certificates.
Because the formulation requires only a pre-activation dense layer followed
by a linear output head, and such structures are prevalent in generative
models, exploring whether WARP can be adapted to generative tasks, where
the notion of a classification margin may need to be redefined, is a
promising direction for future work.



\bibliography{colm2026_conference}
\bibliographystyle{colm2026_conference}

\appendix

\section{Iterative Repair Algorithm}
\label{app:algorithm}

Algorithm~\ref{alg:repair} details the iterative QP repair procedure
outlined in \S\ref{sec:iterative}.

\begin{algorithm}[H]
\caption{Iterative Rank-$r$ QP Repair}
\label{alg:repair}
\KwIn{%
  Weights $(W^{(0)}, b, W_c, b_c)$;
  repair set $\mathcal{S}$;
  remain set $\mathcal{P}$;
  hyperparameters $r, \lambda, \rho, \gamma_s, \gamma_h$;
  maximum iterations $T$%
}
\KwOut{Repaired weight $\widetilde{W}$}
\For{$t = 1, 2, \ldots, T$}{
  Compute $U^{(t)}$ via truncated SVD of $W^{(t-1)}$\;
  Construct coefficient vectors $\{a_i\}$ and $\{p_p\}$
    via Eqs.~\eqref{eq:ai_def}--\eqref{eq:pp_def}\;
  Solve QP~\eqref{eq:qp_obj}--\eqref{eq:c3} to obtain $\beta^*$\;
  $W^{(t)} \leftarrow W^{(t-1)}
    + U^{(t)}\,\operatorname{reshape}(\beta^*,\,r,\,d_{\mathrm{in}})$\;
  \If{$\mathrm{gap}(v_i,\, y_i;\, W^{(t)}) \geq \gamma_s
      \;\;\forall\, i \in [n_s]$}{
    \textbf{break}\;
  }
}
\Return{$\widetilde{W} = W^{(t)}$}
\end{algorithm}

\section{Proofs of Theorems~\ref{thm:repair},~\ref{thm:remain},
and~\ref{thm:lipschitz}}
\label{app:proofs}

\begin{theorem}[Repair Set Gap, restated]
If Algorithm~\ref{alg:repair} converges before iteration $T$, then
$\mathrm{gap}(v_i, y_i;\,\widetilde{W})\geq\gamma_s$
for all $i\in\mathcal{S}$.
\end{theorem}
\begin{proof}
The algorithm terminates only when the convergence condition
$\mathrm{gap}(v_i, y_i;\,W^{(t)}) \geq \gamma_s$ for all $i \in \mathcal{S}$
is satisfied, which directly encodes this inequality for every $i \in \mathcal{S}$.
\end{proof}

\begin{theorem}[Remain Set Gap, restated]
At any iterate $W^{(t)}$ produced by Algorithm~\ref{alg:repair},
if the QP constraints~\eqref{eq:c2} are satisfied, then
$\mathrm{gap}(v_p,\,\hat{y}_p^{(0)};\,W^{(t)}) \geq \gamma_h$
for all $p \in \mathcal{P}$.
\end{theorem}
\begin{proof}
Constraint~\eqref{eq:c2} requires
$\langle p_p, \beta \rangle \geq \gamma_h -
\mathrm{gap}(v_p, \hat{y}_p; W^{(t-1)})$
at every iteration. Substituting the first-order
approximation~\eqref{eq:gap_linear}, this is equivalent to enforcing
$\mathrm{gap}(v_p, \hat{y}_p; W^{(t)}) \geq \gamma_h$ at each step.
\end{proof}

\begin{theorem}[Lipschitz Robustness Radius, restated]
For any repaired sample $v_i \in \mathcal{S}$, define
$\varepsilon_i^* = \mathrm{gap}(v_i, y_i;\,\widetilde{W}) / 2L$.
Then for any perturbation $\delta$ with $\|\delta\|_2 \leq
\varepsilon_i^*$, the predicted label of the repaired classifier at
$v_i + \delta$ remains $y_i$.
\end{theorem}
\begin{proof}
By the Lipschitz bound~\eqref{eq:lipschitz},
\begin{align*}
  \mathrm{gap}(v_i + \delta,\, y_i;\,\widetilde{W})
  &\;\geq\; \mathrm{gap}(v_i,\, y_i;\,\widetilde{W}) - L\,\|\delta\|_2 \\
  &\;\geq\; \mathrm{gap}(v_i,\, y_i;\,\widetilde{W})
            - \frac{\mathrm{gap}(v_i,\, y_i;\,\widetilde{W})}{2} \\
  &\;=\; \frac{\mathrm{gap}(v_i,\, y_i;\,\widetilde{W})}{2}
  \;>\; 0.
\end{align*}
A positive gap implies $s_{y_i} > \max_{c \neq y_i} s_c(v_i + \delta)$,
so the predicted label is $y_i$.
\end{proof}


\section{Hyperparameters}
\label{app:hyperparams}

Tables~\ref{tab:qp_hyperparams} and~\ref{tab:gsn_hyperparams} list the
hyperparameters used for QP repair and GSN-FT respectively. All values
are shared across the four task--model configurations.

\begin{table}[h]
  \centering
  \caption{QP repair hyperparameters.}
  \label{tab:qp_hyperparams}
  \begin{tabular}{clc}
    \toprule
    Symbol & Description & Value \\
    \midrule
    $r$         & Rank of weight perturbation         & 2    \\
    $\gamma_s$  & Repair set gap margin target         & 1.0  \\
    $\gamma_h$  & Remain set gap margin lower bound    & 0.3  \\
    $\lambda$   & Slack penalty coefficient            & 50.0 \\
    $\rho$      & regularization coefficient  & 2.0  \\
    $T$         & Maximum iterations                   & 300  \\
    ---         & QP solver                            & OSQP~\citep{stellato2020osqp} \\
    \bottomrule
  \end{tabular}
\end{table}

\begin{table}[h]
  \centering
  \caption{GSN-FT hyperparameters.}
  \label{tab:gsn_hyperparams}
  \resizebox{\textwidth}{!}{%
  \begin{tabular}{lc}
    \toprule
    Parameter & Value \\
    \midrule
    Max steps $T_{\mathrm{FT}}$    & 30 \\
    Learning rate                  & $1\times10^{-3}$ \\
    Batch size                     & 32 \\
    Auxiliary set size $N_{\mathrm{aux}}$ & 8 \\
    Trainable parameters           & repair layer $W$ + classification head $W_c$ (592{,}130 total) \\
    Checkpoint selection criterion & step with maximum GSN $\kappa$ \\
    \bottomrule
  \end{tabular}}
\end{table}

\section{Monte Carlo Stress Test and Proximity Analysis}
\label{app:stress_test}

To empirically assess the conservativeness of the Lipschitz robustness
radius $\varepsilon^*$, we conduct an expanding-radius Monte Carlo stress
test: for each repaired sample, 1{,}000 perturbation vectors are drawn
uniformly from progressively expanding $\ell_2$-balls, and the smallest
radius multiplier at which the predicted label first flips is recorded.
Table~\ref{tab:mc_tightness} summarises the results.

\begin{table}[h]
  \centering
  \caption{Expanding-radius Monte Carlo stress test (tightness ratio:
           smallest multiplier $k$ such that a flip is observed within
           $k\times\varepsilon^*$).}
  \label{tab:mc_tightness}
  \resizebox{\textwidth}{!}{%
  \begin{tabular}{lcccccccc}
    \toprule
    & \multicolumn{4}{c}{SST-2} & \multicolumn{4}{c}{RTE} \\
    \cmidrule(lr){2-5}\cmidrule(lr){6-9}
    & \multicolumn{2}{c}{DistilBERT} & \multicolumn{2}{c}{BERT}
    & \multicolumn{2}{c}{DistilBERT} & \multicolumn{2}{c}{BERT} \\
    \cmidrule(lr){2-3}\cmidrule(lr){4-5}\cmidrule(lr){6-7}\cmidrule(lr){8-9}
    Multiplier $k$ & Flipped & Cum.\% & Flipped & Cum.\%
                   & Flipped & Cum.\% & Flipped & Cum.\% \\
    \midrule
    $\leq 1\times$   & 0/108   & 0.0\%   & 0/105   & 0.0\%   & 0/47  & 0.0\%   & 0/50  & 0.0\%   \\
    $\leq 20\times$  & 37/108  & 34.3\%  & 0/105   & 0.0\%   & 13/47 & 27.7\%  & 0/50  & 0.0\%   \\
    $\leq 25\times$  & 78/108  & 72.2\%  & 0/105   & 0.0\%   & 34/47 & 72.3\%  & 0/50  & 0.0\%   \\
    $\leq 35\times$  & 107/108     & 99.1\%     & 0/105   & 0.0\%   & 47/47 & 100.0\% & 0/50  & 0.0\%   \\
    $\leq 40\times$  & 108/108 & 100.0\% & 0/105   & 0.0\%   & ---   & ---     & 0/50  & 0.0\%   \\
    $\leq 150\times$ & ---     & ---     & 7/105   & 6.7\%   & ---   & ---     & 12/50 & 24.0\%  \\
    $\leq 250\times$ & ---     & ---     & 21/105     & 20.0\%     & ---   & ---     & 30/50 & 60.0\%  \\
    $\leq 400\times$ & ---     & ---     & 66/105  & 62.9\%  & ---   & ---     & 46/50 & 92.0\%  \\
    $\leq 600\times$ & ---     & ---     & 104/105     & 99.0\%     & ---   & ---     & 50/50 & 100.0\% \\
    $\leq 700\times$ & ---     & ---     & 105/105 & 100.0\% & ---   & ---     & ---   & ---     \\
    \midrule
    Median tightness & \multicolumn{2}{c}{$25\times$}
                     & \multicolumn{2}{c}{$400\times$}
                     & \multicolumn{2}{c}{$25\times$}
                     & \multicolumn{2}{c}{$250\times$} \\
    \bottomrule
  \end{tabular}}
\end{table}

The Type~A pattern is strikingly consistent across tasks: DistilBERT
exhibits a median tightness ratio of $25\times$ on both SST-2 and RTE,
with all samples flipping by $40\times$ (SST-2) and $35\times$ (RTE).
For Type~B BERT, the certificates are far more conservative: no flip occurs
below $90\times\varepsilon^*$ on RTE (vs.\ $150\times$ on SST-2), with
medians of $250\times$ and $400\times$ respectively.
These results confirm that the Lipschitz bound is sound but conservative,
and that its conservativeness is determined primarily by architecture type.

\paragraph{Real adversarial sample proximity analysis.}
Table~\ref{tab:proximity} reports the accuracy of the repaired DistilBERT
model on SST-2 adversarial samples grouped by their distance to the nearest
repaired embedding, expressed as a multiple of $\varepsilon^*$.
The median distance ratio is $36.9\times$ and accuracy decreases
monotonically with distance, mirroring the $20$--$25\times$ first-flip
boundary from the stress test.
For RTE DistilBERT, the median distance ratio is $44.1\times$, consistent
with the stress test boundary, though the monotone trend is weaker owing to
the smaller number of attack samples per distance band.
For both BERT configurations, all real adversarial samples fall well beyond
the stress-test flip boundary ($150\times$ / $90\times$), a direct
structural consequence of the Type~B architecture inflating $L$ and
compressing $\varepsilon^*$.

\begin{table}[h]
  \centering
  \caption{Accuracy of the repaired DistilBERT model (SST-2) on adversarial
           samples grouped by distance to the nearest repaired embedding
           as a multiple of $\varepsilon^*$.}
  \label{tab:proximity}
  \begin{tabular}{lrr}
    \toprule
    Distance band ($d / \varepsilon^*_{\mathrm{nearest}}$) & Samples & Accuracy \\
    \midrule
    $\leq 20\times$        &    5 & 100.0\% \\
    $20\times$--$25\times$ &   49 &  87.8\% \\
    $25\times$--$30\times$ &  180 &  82.2\% \\
    $30\times$--$35\times$ &  267 &  74.5\% \\
    $35\times$--$40\times$ &  271 &  68.6\% \\
    $40\times$--$50\times$ &  324 &  60.5\% \\
    $> 50\times$           &  111 &  46.8\% \\
    \midrule
    Overall                & 1207 &  68.7\% \\
    \bottomrule
  \end{tabular}
\end{table}

\section{Per-Attack Generalization Breakdown}
\label{app:full_baselines}

Table~\ref{tab:per_attack} reports per-method attack generalization
broken down by individual attack type.

\begin{table}[h]
  \centering
  \caption{Per-attack-method generalization accuracy.}
  \label{tab:per_attack}
  \resizebox{\textwidth}{!}{%
  \begin{tabular}{lllccccccc}
    \toprule
    Task & Model & Method & CheckList & SememePSO & StressTest & T3 & TextBugger & TextFooler & Mean \\
    \midrule
    \multirow{8}{*}{SST-2}
      & \multirow{4}{*}{DistilBERT}
        & Original                    & 8.7\%  & 34.3\% & 6.4\%  & 24.2\% & 29.4\% & 30.9\% & 22.3\% \\
      & & Same-layer LoRA ($r\!=\!2$) & 41.0\% & 50.9\% & 89.1\% & 46.4\% & 55.8\% & 56.9\% & 56.7\% \\
      & & All-layers Full FT          & 59.3\% & 37.7\% & 73.7\% & 38.6\% & 46.2\% & 45.6\% & 50.2\% \\
      & & \textbf{WARP}               & \textbf{70.5\%} & \textbf{60.6\%} & \textbf{92.9\%} & \textbf{64.1\%} & \textbf{59.9\%} & \textbf{66.2\%} & \textbf{69.0\%} \\
    \cmidrule(lr){2-10}
      & \multirow{4}{*}{BERT}
        & Original                    & 8.7\%  & 34.3\% & 11.5\% & 26.8\% & 28.4\% & 29.4\% & 23.2\% \\
      & & Same-layer LoRA ($r\!=\!2$) & 45.7\% & 59.4\% & \textbf{80.8\%} & 62.1\% & 53.8\% & \textbf{58.3\%} & 60.0\% \\
      & & All-layers Full FT          & 62.7\% & 56.0\% & 64.7\% & 48.4\% & 51.8\% & 56.9\% & 56.7\% \\
      & & \textbf{WARP}               & \textbf{66.8\%} & \textbf{60.6\%} & \textbf{80.8\%} & \textbf{64.7\%} & \textbf{59.9\%} & 55.4\% & \textbf{64.7\%} \\
    \midrule
    \multirow{8}{*}{RTE}
      & \multirow{4}{*}{DistilBERT}
        & Original                    & 23.9\% & 29.6\% & 23.8\% & \textbf{60.0\%} & 47.7\% & 39.5\% & 37.4\% \\
      & & Same-layer LoRA ($r\!=\!2$) & 32.6\% & 18.5\% & 16.7\% & 55.0\% & 45.5\% & 25.6\% & 32.3\% \\
      & & All-layers Full FT          & 26.1\% & 33.3\% & 23.8\% & 45.0\% & 45.5\% & \textbf{44.2\%} & 36.3\% \\
      & & \textbf{WARP}               & \textbf{54.3\%} & \textbf{44.4\%} & \textbf{40.5\%} & 45.0\% & \textbf{54.5\%} & 25.6\% & \textbf{44.1\%} \\
    \cmidrule(lr){2-10}
      & \multirow{4}{*}{BERT}
        & Original                    & 23.9\% & 51.9\% & 19.0\% & 65.0\% & 31.8\% & 44.2\% & 39.3\% \\
      & & Same-layer LoRA ($r\!=\!2$) & 23.9\% & \textbf{59.3\%} & 16.7\% & \textbf{70.0\%} & 27.3\% & 46.5\% & 40.6\% \\
      & & All-layers Full FT          & 23.9\% & 44.4\% & 26.2\% & 40.0\% & 40.9\% & \textbf{65.1\%} & 40.1\% \\
      & & \textbf{WARP}               & \textbf{56.5\%} & 37.0\% & \textbf{50.0\%} & 45.0\% & \textbf{45.5\%} & 53.5\% & \textbf{47.9\%} \\
    \bottomrule
  \end{tabular}}
\end{table}

On SST-2, All-layers Full FT achieves 100\% repair accuracy but suffers the
worst attack generalization (50.2\% / 56.7\%), consistent with overfitting
tens of millions of parameters to fewer than 1{,}000 samples.
On RTE, the same pattern holds: Full FT matches WARP on repair accuracy but
trails on attack generalization (36.3\% / 40.1\% vs.\ 44.1\% / 47.9\%),
with no consistent advantage on any individual attack method.
WARP's rank-2, single-layer constraint acts as an implicit regularizer,
achieving the best mean attack generalization while providing formal
certificates that no gradient-based method can offer.

\section{Ablation Studies}
\label{app:ablation}

\subsection*{Effect of Rank $r$}

\begin{table}[h]
  \centering
  \caption{Rank ablation. SST-2: $|\mathcal{S}|=108/105$, $|\mathcal{P}|=800$.
    RTE: $|\mathcal{S}|=47/50$, $|\mathcal{P}|=200$.
    ``Conv'' denotes QP convergence iterations.
    $^*$QP infeasible at rank~1.}
  \label{tab:rank_ablation}
  \resizebox{\textwidth}{!}{%
  \begin{tabular}{lcccccccc}
    \toprule
    & \multicolumn{4}{c}{DistilBERT} & \multicolumn{4}{c}{BERT} \\
    \cmidrule(lr){2-5}\cmidrule(lr){6-9}
    Rank & Atk.\ Gen & Gen.\ Acc & Conv & Time
         & Atk.\ Gen & Gen.\ Acc & Conv & Time \\
    \midrule
    \multicolumn{9}{l}{\textit{SST-2}} \\
    \midrule
    1          & 68.0\% & 91.8\% & 4 & 34s  & 61.7\% & 95.1\% & 7       & 68s  \\
    2          & 69.0\% & 91.4\% & 5 & 33s & 64.7\% & 92.7\% & 6        & 77s \\
    8          & 68.1\% & 92.0\% & 5 & 160s & 59.2\% & 94.0\% & 4       & 108s \\
    32         & 68.7\% & 92.1\% & 8 & 989s & 60.6\% & 93.9\% & 4       & 535s \\
    \midrule
    \multicolumn{9}{l}{\textit{RTE}} \\
    \midrule
    1          & 43.4\% & 83.2\% & 3 & 5s   & 58.6\% & 1.2\%$^*$ & ---$^*$ & 6s  \\
    2          & 44.1\% & 83.8\% & 4 & 9s & 47.9\% & 86.2\% & 6           & 15s \\
    8          & 44.7\% & 76.2\% & 2 & 7s   & 50.3\% & 88.6\% & 4       & 23s  \\
    32         & 45.9\% & 80.8\% & 4 & 122s & 49.1\% & 88.2\% & 4       & 178s \\
    \bottomrule
  \end{tabular}}
\end{table}

As shown in Table~\ref{tab:rank_ablation}, on SST-2, rank has little impact on DistilBERT (spread ${<}1\%$) but
degrades BERT above rank~2 (59--61\% vs.\ 64.7\%).
On RTE, DistilBERT is similarly insensitive (spread ${<}2\%$), while
BERT peaks at rank~2 and degrades at higher ranks; rank~1 is infeasible
for BERT on RTE, further motivating rank~2 as the default.
The degradation at higher ranks reflects that peripheral singular
directions erode first-order approximation quality without a
compensating gain in expressiveness.

\subsection*{Effect of Repair and Remain Set Size}

\begin{table}[h]
  \centering
  \caption{Repair set size ablation ($|\mathcal{P}|$ fixed at 800 for SST-2,
    200 for RTE).}
  \label{tab:repair_size}
  \resizebox{\textwidth}{!}{%
  \begin{tabular}{lcccccccc}
    \toprule
    & \multicolumn{4}{c}{DistilBERT} & \multicolumn{4}{c}{BERT} \\
    \cmidrule(lr){2-5}\cmidrule(lr){6-9}
    $|\mathcal{S}|$ & \multicolumn{2}{c}{SST-2} & \multicolumn{2}{c}{RTE}
                    & \multicolumn{2}{c}{SST-2} & \multicolumn{2}{c}{RTE} \\
    \cmidrule(lr){2-3}\cmidrule(lr){4-5}\cmidrule(lr){6-7}\cmidrule(lr){8-9}
    & Atk.\ Gen & Gen.\ Acc & Atk.\ Gen & Gen.\ Acc
    & Atk.\ Gen & Gen.\ Acc & Atk.\ Gen & Gen.\ Acc \\
    \midrule
    10       & 33.2\% & 97.9\% & 36.4\% & 94.2\% & 29.2\% & 96.2\% & 38.2\% & 97.8\% \\
    25       & 51.0\% & 96.2\% & 39.2\% & 87.0\% & 42.0\% & 95.2\% & 45.9\% & 90.6\% \\
    50  & 60.5\% & 93.2\% & ---    & ---    & 60.6\% & 93.8\% & ---    & ---    \\
    75       & 65.6\% & 92.5\% & ---    & ---    & 62.5\% & 93.0\% & ---    & ---    \\
    Full     & 69.0\% & 91.4\% & 44.1\% & 83.8\% & 64.7\% & 92.7\% & 47.9\% & 86.2\% \\
    \bottomrule
  \end{tabular}}
\end{table}

\begin{table}[h]
  \centering
  \caption{Remain set size ablation ($|\mathcal{S}|$ fixed at full size).
    $\varepsilon^*_{\mathrm{med}}$ reported for DistilBERT only.
    \textbf{Bold} denotes the default setting.}
  \label{tab:remain_size}
  \resizebox{\textwidth}{!}{%
  \begin{tabular}{lcccccccc}
    \toprule
    & \multicolumn{3}{c}{DistilBERT (SST-2)} & \multicolumn{2}{c}{BERT (SST-2)}
    & \multicolumn{2}{c}{DistilBERT (RTE)} & \multicolumn{1}{c}{BERT (RTE)} \\
    \cmidrule(lr){2-4}\cmidrule(lr){5-6}\cmidrule(lr){7-8}\cmidrule(lr){9-9}
    $|\mathcal{P}|$ & Atk.\ Gen & Gen.\ Acc & $\varepsilon^*_{\mathrm{med}}$
                    & Atk.\ Gen & Gen.\ Acc
                    & Atk.\ Gen & Gen.\ Acc
                    & Atk.\ Gen \\
    \midrule
    0             & 79.5\% &  0.8\% & 1.057 & 76.4\% &  1.1\% & 56.5\% & 15.0\% & 58.3\% \\
    50            & 76.9\% & 52.9\% & 0.305 & 73.1\% & 50.7\% & 41.3\% & 61.8\% & 53.1\% \\
    200  & 72.0\% & 79.8\% & 0.255 & 70.7\% & 80.7\% & 44.1\% & 83.8\% & 47.9\% \\
    400           & 70.3\% & 86.5\% & 0.226 & 69.0\% & 89.0\% & ---    & ---    & ---    \\
    800  & 69.0\% & 91.4\% & 0.211 & 64.7\% & 92.7\% & ---    & ---    & ---    \\
    2400          & 64.3\% & 95.8\% & 0.138 & 60.8\% & 95.5\% & ---    & ---    & ---    \\
    \bottomrule
  \end{tabular}}
\end{table}

As shown in Table~\ref{tab:repair_size}, larger repair sets yield substantially better attack generalization: the
gain from smallest to full set is $+36\%$/$+35\%$ on SST-2 and
$+8\%$/$+10\%$ on RTE, at a modest general accuracy cost of $-5\%$
to $-10\%$.

Table~\ref{tab:remain_size} shows that the remain set governs the tension between general accuracy and attack
generalization: increasing $|\mathcal{P}|$ tightens remain constraints,
forcing $\Delta W$ into more conservative directions.
At $|\mathcal{P}|=0$, general accuracy collapses to near zero on both
tasks, confirming that the remain set is a necessary component of the
framework. Scaling up both sets together therefore yields simultaneous gains in attack generalization and general accuracy.

\end{document}